\documentclass[letterpaper,twocolumn, 10pt, conference]{ieeeconf}  

\IEEEoverridecommandlockouts                              

\usepackage{epsfig} 

\usepackage{graphics}
\usepackage{graphicx}
\usepackage{amsmath}
\usepackage{amssymb}
\usepackage{epsfig,wrapfig}
\usepackage{amsmath}
\usepackage{color}
\usepackage{float}
\usepackage{amsmath}
\usepackage{amsfonts}
\usepackage{amssymb}
\usepackage{graphicx}
\usepackage{relsize}
\usepackage{cite}
\usepackage{verbatim}

\usepackage{wrapfig}

\newcommand{\R}{\mathbb{R}}

\def \QED{\hspace*{\fill}$\blacksquare$\medskip}

\DeclareMathOperator{\N}{\mathbb{N}}

\newtheorem{definition}{Definition}
\newtheorem{theorem}{Theorem}
\newtheorem{lemma}[theorem]{Lemma}

\newtheorem{assumption}{Assumption}
\newtheorem{example}{Example}
\newtheorem{case}{Case}
\title{\LARGE \bf
A Class of Logistic Functions for Approximating State-Inclusive Koopman Operators}

\author{Charles A. Johnson$^{1}$, Enoch Yeung$^{2}$
\thanks{$^{1}$ Charles Johnson is with the Information and Decision Algorithms Laboratories, Brigham Young University, Provo, UT 84602, USA
        {\tt\small charles.addisonj@gmail.com}}%
\thanks{$^{2}$Enoch Yeung, is a research scientist in the Computational Analytics Division at Pacific Northwest National Laboratories
        {\tt\small enoch.yeung@pnnl.gov}}%
}

\begin{document}

\maketitle
\pagestyle{plain}

\begin{abstract}
An outstanding challenge in nonlinear systems theory is identification or learning of a given nonlinear system's Koopman operator directly from data or models.  Advances in extended dynamic mode decomposition approaches and machine learning methods have enabled data-driven discovery of Koopman operators, for both continuous and discrete-time systems.   Since Koopman operators are often infinite-dimensional, they are approximated in practice using finite-dimensional systems.  The fidelity and convergence of a given finite-dimensional Koopman approximation is a subject of ongoing research.  In this paper we introduce a class of Koopman observable functions that confer an approximate closure property on their corresponding finite-dimensional approximations of the Koopman operator.  We derive error bounds for the fidelity of this class of observable functions, as well as identify two key learning parameters which can be used to tune performance. We illustrate our approach on two classical nonlinear system models: the Van Der Pol oscillator and the bistable toggle switch. 

\end{abstract}

\section{Introduction}
Koopman operators are a class of models used for understanding important dynamical properties of nonlinear systems \cite{arbabi2016ergodic, budivsic2012applied,tu2013dynamic,mezic2013analysis,mezic2005spectral}.  The Koopman operator captures the behavior of nonlinear systems by representing them as higher-order linear systems in a lifted function space. Though posed originally by Bernard Koopman in the early 1930s \cite{koopman1931hamiltonian}, Koopman operator theory has gained traction and popularity quite recently for their applicability in nonlinear system analysis \cite{mezic2005spectral} and as a data driven method for system identification \cite{kutz2016dynamic}, \cite{williams2016extending}.  


In specific instances, a nonlinear system may admit a finite dimensional Koopman operator.  When a Koopman operator is finite dimensional, the evolution of all the system's states are completely characterized as a linear combination of a finite number of lifting or observable functions. In general, Koopman operators can be typically infinite dimensional and have continuous or countably infinite spectra \cite{budivsic2012applied}. Infinite dimensional operators are computationally unwieldy, which motivate the development of learning methods for high fidelity {\it finite dimensional} approximations of the Koopman operator.    


Methods for learning Koopman operators have existed since Torsten Carleman developed a technique for the linearization of  nonlinear systems in finite dimensional representations \cite{carleman1932application}. However Carleman linearization only results in a finite dimensional representation for systems possessing specific internal structure, e.g. feedforward dynamics.  More recently work on finding approximations to Koopman operators include state-of-the-art techniques such as dynamic mode decomposition (DMD) \cite{schmid2010dynamic} and extended dynamic mode decomposition (EDMD)\cite{mezic2004comparison, williams2015data,mezic2005spectral}.  EDMD, in particular, is effective, since it learns the space of Koopman observables using a dictionary of generic basis functions \cite{williams2015data}.  

In particular, recent work has shown the effectiveness of deep \cite{2017arXiv170806850Y} and shallow neural networks \cite{li2017extended} for learning approximate Koopman operators.  However, these methods essentially learn black box dictionaries, where the relationship between the  mathematical form of the neural network outputs and the actual system are unclear.    \cite{2017arXiv170806850Y} showed that neural networks can learn smooth dictionaries consisting primarily of observables with a sigmoidal response profile.   One contribution of this paper is to clarify the value of sigmoidal basis functions in learning approximate Koopman operators. 

Most generally, we consider the problem of learning a finite dimensional approximation to the Koopman generator for a known continuous nonlinear system.  We first identify a class of Koopman basis functions that produces an exact Koopman realization of the same dimension as the original system.  We show however, that these bases functions provide virtually no insight into the stability of the underlying system, since they exclude the system's actual state.

Motivated by these findings, we consider Koopman observables that include the system state. However, most state inclusive Koopman observable spaces suffer from exploding dimensionality. We introduce the property of {\it finite approximate closure}, namely the ability of a state-inclusive Koopman basis with finite cardinality to simultaneously approximate a nonlinear system's dynamics as well as its own flow. We show that Koopman observables obtained from a special class of multi-variate logistic functions, satisfy an approximation property we define as finite approximate closure.  We derive explicit error bounds and show its relationship with learning parameters describing the dictionary resolution and ultra-sensitivity.  We show this class of dictionary functions learns the dynamics of the Van der Pol oscillator and  the bistable toggle switch. 

The rest of this paper is organized as follows.  Section \ref{sec:KLP} introduces the problem of learning a finite dimensional approximation of the so-called Koopman operator.  Section \ref{sec:KBases} motivates the use of Koopman bases that include states of the original system.  Section \ref{sec:SILL} introduces the concept of finite approximate closure and a class of state-inclusive Koopman bases functions that satisfy finite approximate closure.  Section \ref{sec:SILLerror} derives error bounds for this special Koopman basis and Section \ref{sec:examples} illustrates the accuracy of these bases functions for learning the dynamics of two nonlinear systems: the Van Der Pol oscillator and the toggle switch. 

\section{The Koopman Generator Learning Problem}\label{sec:KLP}
Consider a nonlinear system with dynamics
\begin{equation}\label{eq:nonlinear_system}
\dot{x} = f(x)
\end{equation}
where $x \in M \subset \mathbb{R}^n$, $f:\mathbb{R}^n \rightarrow \mathbb{R}^n$ is continuously differentiable, time-invariant, and nonlinear in $x$.  Let $x_0$ denote the initial condition $x(0)$ for the system and $M$ denote the state-space of the dynamical system.  
We introduce the concepts of a Koopman generator and its associated multi-variate Koopman semigroup, following the exposition of \cite{budivsic2012applied}. 

\subsection{The Koopman Generator}

For continuous nonlinear systems, the Koopman semigroup is a semigroup ${\cal K}_{t\in \mathbb{R}}$ of linear but infinite dimensional operators ${\cal K}_t$ that acts on a space of  functions $\psi: M \rightarrow \mathbb{R}$, often referred to as {\it observables}. Each observable function $\psi \in \Psi$, where $\Psi$ is a finite or infinite dimensional function space.   We thus say ${\cal K}_t:\Psi \rightarrow \Psi$ is an operator for each $t\geq 0$.   The Koopman semigroup provides an alternative view for evolving the state of a dynamical system: 
\begin{equation} \mathcal{K}_t \circ \psi(x_0) = \psi \circ \Phi_t(x_0),\end{equation}
where $\Phi_t(x_0)$ is the flow map of the dynamical system (\ref{eq:nonlinear_system}), evolved forward up to time $t$, given the initial condition $x_0$.

Instead of examining the evolution of the state of a dynamical system, the Koopman semigroup allows us to study the forward evolution of observables defined on the state of a dynamical system \cite{williams2015data}. 

The generator ${\cal K_G}$ for the Koopman semigroup is defined as 
\begin{equation}
{\cal K}_G \psi  \equiv \lim_{t\rightarrow 0} \frac{ {\cal K}_t \psi - \psi }{t}
\end{equation}
\begin{lemma}
The Koopman generator ${\cal K}_G$ is a linear operator.
\end{lemma}
\begin{proof}
Notice that the transformation ${\cal K}_G:\Psi \rightarrow \Psi$ is an operator, since each ${\cal K}_t$ is an operator on $\Psi.$   Moreover, the algebraic limit theorem and the linearity of each ${\cal K}_t$ guarantees linearity of ${\cal K}_G$, which implies it is a linear operator.  
\end{proof}

In general, ${\cal K}_G$ may not have a finite or countably infinite-dimensional matrix representation, since the limit of the spectrum of ${\cal K}_{t\geq 0}$ as $t \rightarrow 0$ may be continuous and therefore uncountably infinite, see \cite{mezic2005spectral} for a thorough study of several examples.     
\subsection{Problem Statement}
We restrict our attention in this paper to systems with finite or countably infinite dimensional Koopman generators ${\cal K}_G$.  Given such a continuous nonlinear dynamical system, specifically $f(x)$ and $x$ from (\ref{eq:nonlinear_system}); we aim to  learn $\psi(x)$ and Koopman generator ${\cal K}_G$ to solve the optimization problem
\begin{equation}\label{eq:objective}
\min_{{\cal K_G}, \psi \in \Psi} \left\Vert  \frac{d\psi(x(t))}{dt}  - {\cal K}_G \psi(x(t)) \right\Vert
\end{equation}

This optimization problem is often non-convex, since the form of $\psi(x(t))$ is unknown or parametrically undefined. Both the Koopman generator and the basis functions must be discovered simultaneously to minimize the above objective function.  This is true in data-driven formulations of the problem  where $f(x)$ is completely unknown. Additionally, it is true in learning problems  where $f$ is known but ${\cal K}_G$ has yet to be discovered.

A solution pair $({\cal K}_G, \psi(x))$ that achieves exactly zero error is an exact realization of a Koopman generator and its associated observable function. In general, there may be multiple solutions that achieve exactly zero error.  To see this, note that if 
\begin{equation}
 \frac{d \psi(x(t))}{dt} = {\cal K}_G \psi(x(t)) 
\end{equation}
then a state transformation $\varphi(x(t)) = T^{-1} \psi(x(t))$ also defines an exact solution pair 
$(T{\cal K}_GT^{-1},T^{-1} \psi(x(t))).$

Solving for an exact solution pair $({\cal K}_G,\psi(x) )$ in practice may be difficult for at least three reasons.  First, evaluating $\frac{d\psi(x(t)))}{dt}$ requires numerical differentiation, which incurs a certain degree of numerical error. Second, ${\cal K}_G$ may be infinite dimensional and the collection of observable functions $\Psi \equiv \{\psi_1, \psi_2 ,...\psi_{n_L} \}, n_L \leq \infty$ is unknown {\it a priori}. 

We refer to any solution pair $({\cal K}_G, \psi(x) )$ that results in a non-zero error as an approximate solution.  Note that, given vector norm $||\cdot||$, the error for any approximate solution may be specific to a particular $\epsilon(t)$, of the form 
\begin{equation}
\epsilon(x) = \left\Vert\frac{d \psi(x(t))}{dt}  - {\cal K}_G \psi(x(t)) \right\Vert 
\end{equation}
and thus may vary as a function of $x.$   We seek the best approximation that minimizes $\epsilon(x)$ over all $x \in \mathbb{R}^n.$

The goal of Koopman generator learning is thus to obtain a ``lifted" linear representation of a nonlinear system, defined on a set of observable functions, that enables direct application of the rich body of techniques for analyzing linear systems.  Even if it is only possible to identify an approximate solution that minimizes $\epsilon(x) < M$ , for all $x \in {\cal P}$; spectral analysis of the system can provide insight into the stability of the underlying nonlinear system within the region ${\cal P}$ of the phase space.    


\section{Selection of Koopman Basis Functions}\label{sec:KBases}
The standard approach for learning ${\cal K}_G$ and $\Psi$ is to first postulate a set of dictionary functions that approximate and span $\Psi$ or a subspace of $\Psi$ and second, estimate ${\cal K}_G$ given fixed $\psi(x).$   This approach is known as extended dynamic mode decomposition.   The technique involves constructing a set of dictionary functions $\Psi^D =\{\psi_1, ..., \psi_{N_D}\},$ evaluating the dictionary over a time-shifted stack of state trajectories 
\begin{equation*}
X_p = \begin{bmatrix} x(t_{n+1}) & \hdots & x(t_{0}) \end{bmatrix},\mbox { \hspace{0.1mm} } X_f = \begin{bmatrix} x(t_n) & \hdots & x(t_{1}) \end{bmatrix} 
\end{equation*}
to obtain 
\begin{equation} \begin{aligned}
\begin{array}{ccc}
\Psi(X_f)  &=  \left[\begin{array}{c|c|c}\psi(x^{(0)}_{n+1})  &  \hdots  &  \psi(x^{(0)}_1) \\ \vdots  &  \ddots  &  \vdots \\ \psi(x^{(p)}_{n+1}) & \hdots & \psi(x^{(p)}_1 )  \end{array} \right]   \\
\Psi(X_p)  & =  \left[\begin{array}{c|c|c}\psi(x^{(0)}_{n})  &  \hdots  &  \psi(x^{(0)}_0) \\ \vdots  &  \ddots  &  \vdots \\ \psi(x^{(N_D)}_{n}) & \hdots & \psi(x^{(N_D)}_0 )  \end{array} \right]. \end{array}
\end{aligned} \end{equation}
and approximating the Koopman operator by minimizing the (regularized) objective function
\begin{equation}\label{eq:edmd}
|| \Psi(X_f) - K \Psi(X_p) ||_2  + \zeta ||K||_{2,1} 
\end{equation}
where $K$ is the finite approximation to the true Koopman operator, ${\cal K}$, for a discrete time system, and $||{\cal K}||_{2,1}$ is the 1-norm of the vector of 2-norms of each column of ${\cal K}.$

Note that $\zeta =0$ provides the classical formulation for extended dynamic mode decomposition.   As suggested by the notation, this approach is most commonly applied in the study of open-loop nonlinear discrete-time dynamical systems, see \cite{williams2015data,li2017extended,kutz2016dynamic} for several examples.   More recently, \cite{proctor2016dynamic,williams2016extending,proctor2016generalizing} illustrated the ability to learn control Koopman operators for discrete time dynamical systems, which introduce a new class of control-Koopman learning problems. 

The Koopman generator learning problem is the continuous time analogue of minimizing (\ref{eq:edmd}).  However, when preforming Koopman learning in a continuous dynamical system, the matrix $\Psi(X_f)$ must be replaced with a finite-difference approximation of the derivative matrix $\frac{d}{dt} \Psi(X_f).$  So, in general, the accuracy of this estimate is sensitive to the finite-difference approximation used. 

The purpose of this paper is to study the quality of a class of observable functions for estimating a Koopman generator in finite dimensions.  To this end, we restrict our attention to Koopman generator learning problems where the underlying function $f(x)$ is known, but difficult to analyze using local linearization methods. 
\begin{assumption}
Given a nonlinear system (\ref{eq:nonlinear_system}), we suppose that the underlying vector field $f(x)$ is known. 
\end{assumption}
This allows us to evaluate the quality of a class of observable functions, independent of the error imposed by any finite-difference scheme for estimating the derivative 
\[\frac{d}{dt} \Psi(X_f).\]
This leaves us with two challenges. First, identifying a suitable dictionary of observables or lifting functions from which to construct the observable functions $\psi(x(t)).$  Second, identifying or estimating ${\cal K}_G$, given $f(x)$.   

\subsection{Understanding Stability with Koopman Observables}
Not all Koopman observables (or what we will refer to as Koopman liftings) yield insight into the stability of the underlying system.  For example, suppose that we are given a nonlinear system of the form (\ref{eq:nonlinear_system}).  Further suppose that $f$ is invertible on ${\cal M}\subset \mathbb{R}^n.$ Then 
the Koopman generator, $\mathcal{K}_G$ must satisfy: 
\begin{equation}
\frac{d\psi(x)}{dt} = \mathcal{K}_G\psi(x)
\end{equation}
Let $c \in \mathbb{R}$ be an arbitrary constant. First, we choose a set of functions $\{ \mathcal{L}_1, \mathcal{L}_2, \hdots,\mathcal{L}_n | \mathcal{L}_i: \R^n\rightarrow \R, \forall i\in \{1, 2, \hdots,n\}\}$ so that:
\begin{equation}
\mathcal{L}_i(x)= \int_0^x cf^{-1}(\tau)d\tau
\end{equation}
 Notice that the Koopman observable function $\psi(x)$ defined as
\begin{equation}
\psi(x) \equiv \begin{bmatrix}
\psi_1(x)\\
\psi_2(x)\\
\vdots \\
\psi_n(x)
\end{bmatrix} \equiv \begin{bmatrix}
e^{\mathcal{L}_1(x)}\\
e^{\mathcal{L}_2(x)}\\
\vdots \\
e^{\mathcal{L}_n(x)}
\end{bmatrix}
\end{equation}
has time-derivative that can be expressed in state-space form as 
\begin{equation}
\begin{aligned}
\frac{d}{dt} \psi(x(t)) & = 
\frac{d}{dt}\begin{bmatrix}
e^{\mathcal{L}_1(x)}\\
e^{\mathcal{L}_2(x)}\\
\vdots \\
e^{\mathcal{L}_n(x)}
\end{bmatrix}   
= \begin{bmatrix}
e^{\mathcal{L}_1(x)}\\
e^{\mathcal{L}_2(x)}\\
\vdots \\
e^{\mathcal{L}_n(x)}
\end{bmatrix} cf^{-1}(x)f(x)  
\\&=  cI\begin{bmatrix}
e^{\mathcal{L}_1(x)}\\
e^{\mathcal{L}_2(x)}\\
\vdots \\
e^{\mathcal{L}_n(x)}
\end{bmatrix} 
 = K_G \psi(x)
\end{aligned}
\end{equation}
where the Koopman generator is $\mathcal{K}_G = cI.$  Since $c$ is arbitrary, the system can be either stable or unstable, depending on the sign of the choice of $c$.  

Our choice of observable functions provides an exact solution to the Koopman generator learning problem.  The spectral properties of ${\cal K}_G$ are easy to explore as each eigenvalue is simply equal to $c$.  However, this result is uninformative since the stability properties of the Koopman generator are totally dependent on an arbitrary constant and are therefore completely divorced from the vector field $f$.

The key property that is lacking in the above example is the inclusion of the underlying state dynamics in the observable function $\psi(x)$.  Whenever $x$ is contained within $\psi$, this guarantees that any Koopman generator ${\cal K}_G$ and its associated Koopman semigroup $\{\cal K\}_{t \geq 0}$, not only describes the time evolution of $\psi(x)$ but also the underlying system.  Specifically, if $\psi(x) = \left(\psi_1(x), \psi_2(x) ,..., \psi_{n_L}(x)\right)$  contains a $\psi_j(x) = x$ that is the so-called full state observable function, then by definition, \[ d\psi_j(x)/dt = f(x).\]   Thus, the spectrum of ${\cal K}_G$ will characterize the stability of $\psi(x)$, including $\psi_j(x)$.   When $d\psi_j(x)/dt = f(x)  \in \text{span}\{\psi_1,...,\psi_{n_L} \}$ and $d \psi(x)/dt = {\cal K_G} \psi(x)$, we say the system has {\it finite exact closure}.  That is, derivatives of the full-state observable $\psi_j(x) = x$  and the rest of $\psi_i(x), i \neq j$ can be described entirely in terms of the state vector $\psi(x).$  This property does not hold for many nonlinear candidate observable functions.  We give an example: 

\begin{example}
Consider a scalar nonlinear system of the form 
\begin{equation}
\dot{x} = f(x) = -x^2
\end{equation}
First, consider a candidate observable function $\psi(x) = \left( 1, x, x^2\right).$  We want to see if 
\begin{equation}
\dot{\psi}(x)  = {\cal K}_G\psi(x)
\end{equation}
for some ${\cal K}_G.$  Calculating explicitly, we get 
\begin{equation}
\dot{\psi}(x) = \begin{bmatrix} 0 \\ -\psi_3(x) \\ -2\psi_2(x)\psi_3(x)\end{bmatrix}
\end{equation}
The issue is that including $x$ in the state requires including $f(x)$ as part of the derivative, which implies that each time $x$ multiplies $f(x)$, you obtain a cubic term which is not included in $\psi(x)$.  Similarly, including cubic terms in $\psi(x)$ results in quartic terms and so on.   This is an example of a system where $\psi(x)$ defined above, does not satisfy finite exact closure.  This is not to say that the system can not be expressed with finite closure, but that our proposed observable function $\psi$ does not satisfy the finite exact closure property.   
\end{example}

\subsection{Finite Approximate Closure}

In general, systems that are not globally topologically conjugate to a finite dimensional linear system, e.g. systems with multiple fixed points, cannot be represented exactly by a finite-dimensional linear Koopman operator that includes the state as part of the set of observables \cite{brunton2016koopman}.  

However, it may be possible to learn a Koopman observable function $\Psi(x)$ that {\it  approximately} satisfies finite closure, defined as follows: \begin{definition} 
Let $\Psi(x):M \rightarrow \mathbb{R}^{N_L}$ where $N_L < \infty$.  
We say $\Psi(x)$ achieves finite $\epsilon$-closure or {\it finite approximate closure} with $O(\epsilon)$ error if and only if  there exists an ${\cal K}_G \in \mathbb{R}^{n\times n}$ and $\epsilon > 0$ such that 
\begin{equation}
\frac{d}{dt} \left(\Psi(x)  \right) = {\cal K}_G \psi(x) + \epsilon(x) 
\end{equation}
We say that $\Psi(x)$ achieves uniform finite approximate closure for some set ${\cal P} \subset M$ if and only if it achieves finite approximate closure with $|\epsilon(x)| < B \in \mathbb{R}$ for all $x \in {\cal P}.$

\end{definition}

Finite approximate closure is a desirable property since, as $\epsilon \rightarrow 0$,  we may use $\mathcal{K}_G$ to preform high fidelity stability, observability and spectral analysis.  For example, if $\epsilon$ is small enough over all x(t) in M, one could study the target trajectory of $x(t)$ given $x_0$ by studying the evolution of a state-inclusive lifting of observable functions, or  $\dot{\psi}(x)  = \mathcal{K}_G \psi(x)$.  Projecting from  $\psi$ to  $x$ is trivial and it's trajectory, an approximation to $x(t)$, may yield stability insights.

By a similar token we also may consider observability analysis and state prediction problems \cite{surana2016linear,vaidya2007observability}. Given a series of measurements with corruption in the model and noise in the measurements, can one predict the state of the system? Under the condition of finite approximate closure and a sufficiently small $\epsilon$ the error of state estimation on the state inclusive lifting of the system (evolving according to the linear relation given by $\mathcal{K}_G$) should also be small.  For more extensive treatment in the use of Koopman operators in the state prediction problem (in discrete time) see \cite{surana2016linear}.

Finally we note that given a matrix $A$ with a spectrum $\sigma_A$, if one adds a perturbation matrix, $P$, where $\Vert P\Vert< \varepsilon_1$, there are established limits on how the spectrum of $A+P$, $\sigma_{A+P}$,  will vary from $\sigma_A$. For example, there are the bounds established in the Hoffman-Wielandt theorem.  So the spectrum of a weakly perturbed matrix is weakly altered. Therefore, if $\mathcal{K}_G^*$ is a close approximation to the true Koopman generator of a system, we can estimate the spectral distribution of the true Koopman generator, including its principal modes and eigenvalues \cite{mezic2005spectral}.  Finite approximate closure of $\mathcal{K}_G^*$ of order $\epsilon$  guarantees bounded error between $\mathcal{K}_G^*$ and an ideal Koopman generator.  Moreover, certain learning parameters can be tuned to arbitrarily reduce the size of $\epsilon$.

\section{State Inclusive Logistic Lifting (SILL) Functions and Finite Approximate Closure} \label{sec:SILL}
To develop an approximation to $\mathcal{K}_G$ we introduce a new class of conjunctive logistic functions. We do so for several reasons. Firstly, logistic functions have well established  functional approximation properties \cite{ito1992approximation}.  Secondly, we now show that sets of logistic functions in this class of models, satisfying a total order, satisfy finite approximate closure.

\begin{figure}
\includegraphics[width=250pt]{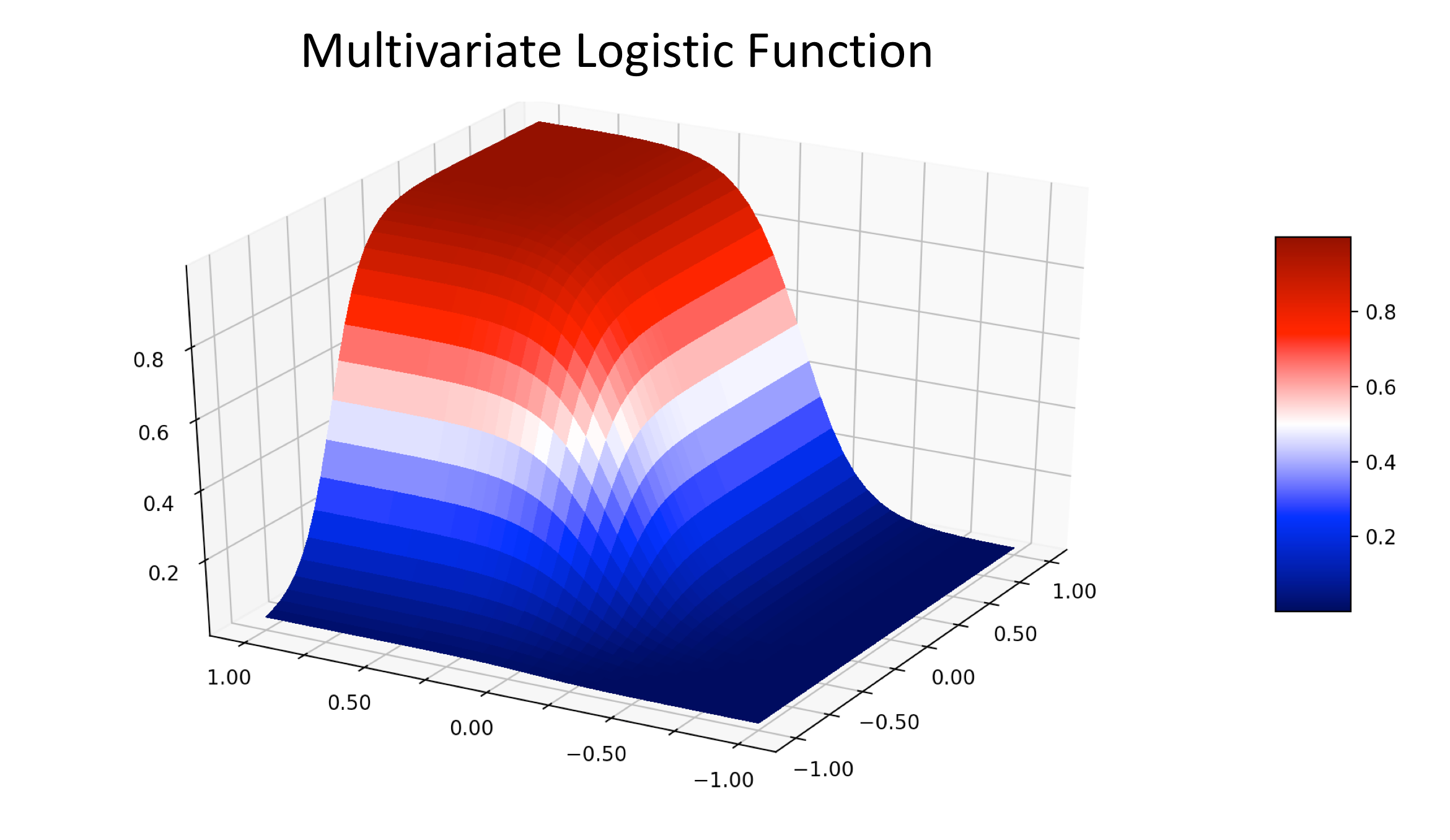}
\caption{This graphic demonstrates an example of our a 2 variable product of logistic functions.  This function would be an example of one of the lifting functions in our proposed lifting scheme for a 2-state system.}
\label{fig:2var_sig_func}
\end{figure}



We define a multivariate conjunctive logistic function as follows:
\begin{equation}
\Lambda_{v_l} (x) \equiv \prod_{i=1}^{n}\lambda_{\mu_i}(x_i)
\end{equation}
where $x \in \mathbb{R}^n$, $ v = (\mu^l_1,...,\mu^l_n),$ and the logistic function $\lambda_{\mu}(x)$ is defined as 
\begin{equation}\label{logistic}
\lambda_{\mu}(x) \equiv \frac{1}{1+e^{-\alpha (x-\mu)}}.
\end{equation}
The parameters $\mu_i$ define the centers or the point of activation along dimension $x_i, i = 1,...,n.$  The parameter $\alpha$ is the steepness parameter, or sensitivity parameter, and determines the steepness of the logistic curve.   Given $N_L$ multivariate logistic functions, we then define a {\it state inclusive logistic lifting} function as $\psi : \R^n\rightarrow \R^{1+n+N_L}$ so that:
\begin{equation}
\psi(x) \equiv \begin{bmatrix}
1\\
x\\
\Lambda
\end{bmatrix}
\end{equation}
where $\Lambda = [\Lambda_{v_1}, \Lambda_{v_2},\hdots,\Lambda_{v_{N_L}} ]^T(x)$. We then have that $\mathcal{K}_G\in\R^{1+n+N_L\times 1+n+N_L}$.  
We first suppose there exists vectors $\{ {\bf w}_i \in\R^{N_L}| \forall i\in\{1, 2, \hdots, N_L\} \}$,   that $f$ can be well approximated by logistic functions \cite{ito1992approximation}, as follows:
\begin{equation}\label{eq:f_regression}
f(x) \approxeq \sum_{l=1}^{N_L} {\bf w}_l \Lambda_{v_l}(x)   
\end{equation}
 This is a fair assumption since the number of logistic functions can be increased until the accuracy of (\ref{eq:f_regression}) is satisfactory.  This accuracy depends on a mesh resolution parameter, which we refer to as $\epsilon$.  This is also generally true of any candidate dictionary for generating Koopman observable functions, e.g. Hermite polynomials, Legendre polynomials, radial basis functions, etc. 

The critical property that enables a high fidelity finite approximate Koopman operator is finite approximate closure. We must show that the time-derivative of these functions can be expressed (approximately) recursively.
The derivative of this multivariate logistic function $\Lambda_{v_l}(x)\in\mathbb{R}$ is given as
\begin{equation}
\dot{\Lambda}_{v_l}(x) = \left(\nabla_x\Lambda_{v_l}(x)\right)^T\frac{\partial x}{\partial t} = \left(\nabla_x \Lambda_{v_l}(x)\right)^T f(x) 
\end{equation}
where the $i^{\text{th}}$ term of the gradient of $\Lambda_{v_l}(x)$ is expressed as 
\begin{equation}
\begin{aligned}
\left[\nabla_x \Lambda_{v_l}(x) \right]_i &=\alpha(\lambda_{\mu^l_i}(x_i) - \lambda_{\mu^l_i}(x_i)^2) \frac{\Lambda_{v_l}(x)}{\lambda_{\mu^l_i}(x_i)} \\
&= \alpha(1 - \lambda_{\mu^l_i}(x_i)) \Lambda_{v_l}(x)\\  
\end{aligned}
\end{equation}
Notice that the time-derivative of $\Lambda_{v_l}$ can be expressed as
\begin{equation}\label{mv_logistic_prime}
\begin{aligned}
\dot{\Lambda}_{v_l}(x)& = \sum_{i=1}^{n} \alpha(1 - \lambda_{\mu^l_i}(x_i)) \Lambda_{v_l}(x) f_i(x)\\ 
&= \sum_{i=1}^{n} \alpha(1 - \lambda_{\mu^l_i}(x_i)) \Lambda_{v_l}(x) \sum_{k=1}^{N_L}w_{ik} \Lambda_{v_k}(x)\\ 
& = \sum_{i=1}^{n}\sum_{k=1}^{N_L} \alpha(1 - \lambda_{\mu^l_i}(x_i)) w_{ik}  \Lambda_{v_l}(x) \Lambda_{v_k}(x)
\end{aligned}
\end{equation}
Thus, we have that the derivative of our multivariate logistic function is a sum of products of logistic functions with a number of predetermined centers.  
There is one critical property that must be satisfied in order to achieve finite approximate closure:
\begin{assumption}
There exists an total order on the set of conjunctive logistic functions $\Lambda_{v_1(x)} ,..., \Lambda_{v_{N_L}}(x)$, induced by the positive orthant $R^n_+$, where $v^l \gtrsim v^k$ whenever $v^k - v^l \in \mathbb{R}^n_+$. 
\end{assumption}
This assumption is satisfied whenever the conjunctive logistic functions are constructed from a evenly spaced mesh grid of points $v_1,...,v_{N_L}$ .  For the purposes of this paper, we will consider evenly spaced mesh grids.  We leave the study of algorithms for learning sparse conjunctive logistic bases for future work. 

Since we have imposed a total order on our logistic basis functions $\mu_l \lesssim \mu_k$ whenever $l \leq k$, we have that the derivative of $\Lambda_{v_{max}(l,k)}$ is the derivative of $\Lambda_{v_{k}}.$ 
Thus we can write
\begin{equation}
\begin{aligned}
\frac{d\Lambda_{\mu_l}}{dt}=  & \sum_{i=1}^{n}\sum_{k=1}^{N_L} \alpha(1 - \lambda_{\mu^l_i}(x_i)) w_{ik}  \Lambda_{v_l}(x) \Lambda_{v_k}(x) \\ 
& \approxeq \sum_{i=1}^{n}\sum_{k=1}^{N_L} \alpha(1 - \lambda_{\mu^l_i}(x_i)) w_{ik}  \Lambda_{v_{max}(l,k)}(x) 
\end{aligned}
\end{equation}
where 
\begin{equation}
v_{max}(l,k) = \left( \max\{\mu^l_1,\mu^k_1\}, ..., \max\{\mu^l_n,\mu^k_n\} \right).
\end{equation}
which shows that conjunctive logistic functions satisfy finite approximate closure. 

\begin{figure}
\includegraphics[width=250pt]{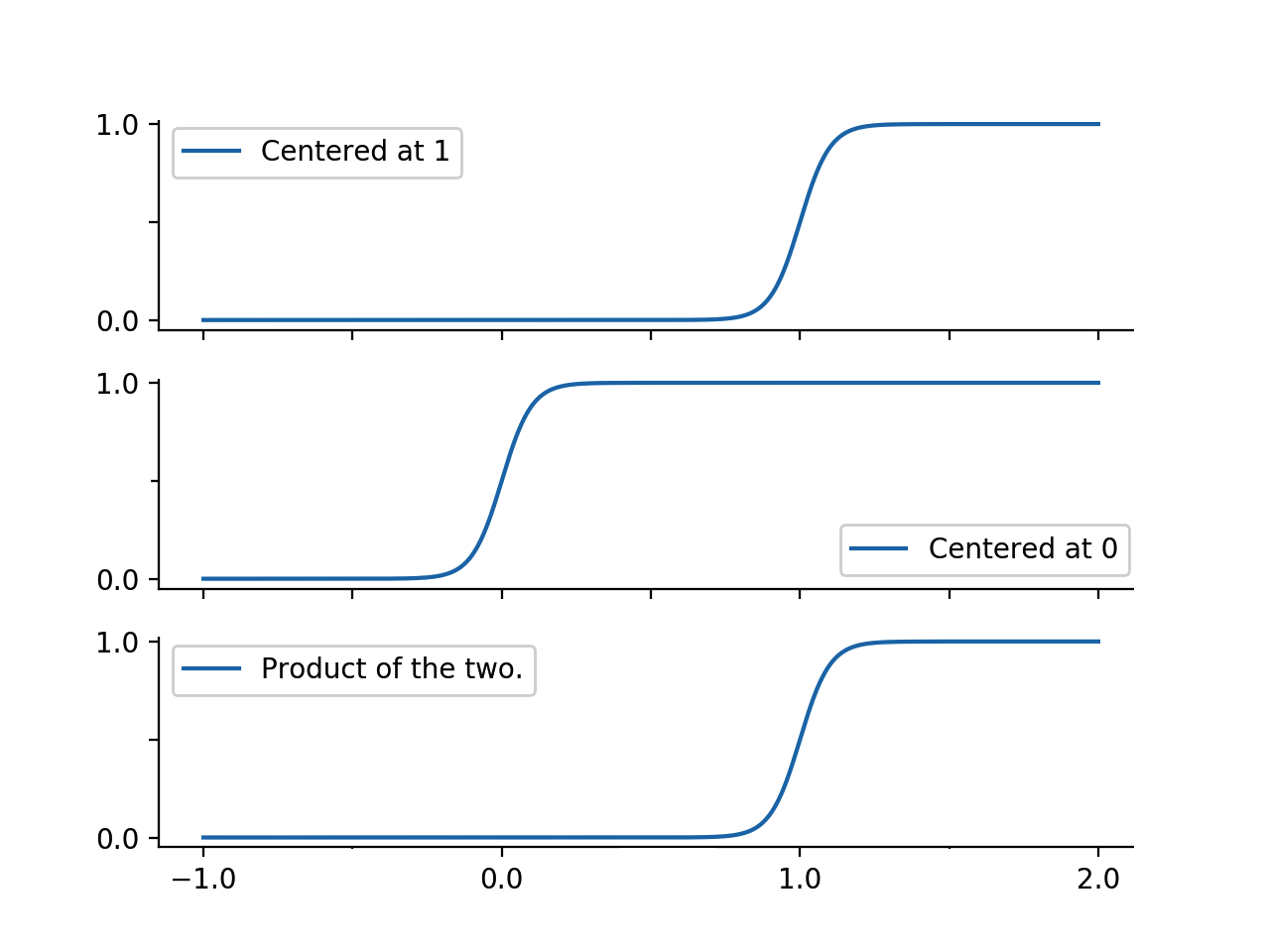}
\caption{This graphic demonstrates how a product of logistic functions may be approximated by the logistic function with the rightmost center.}
\label{fig:prod_of_logistics}
\end{figure}
\begin{figure}
\centering
\includegraphics[width=0.75\columnwidth]{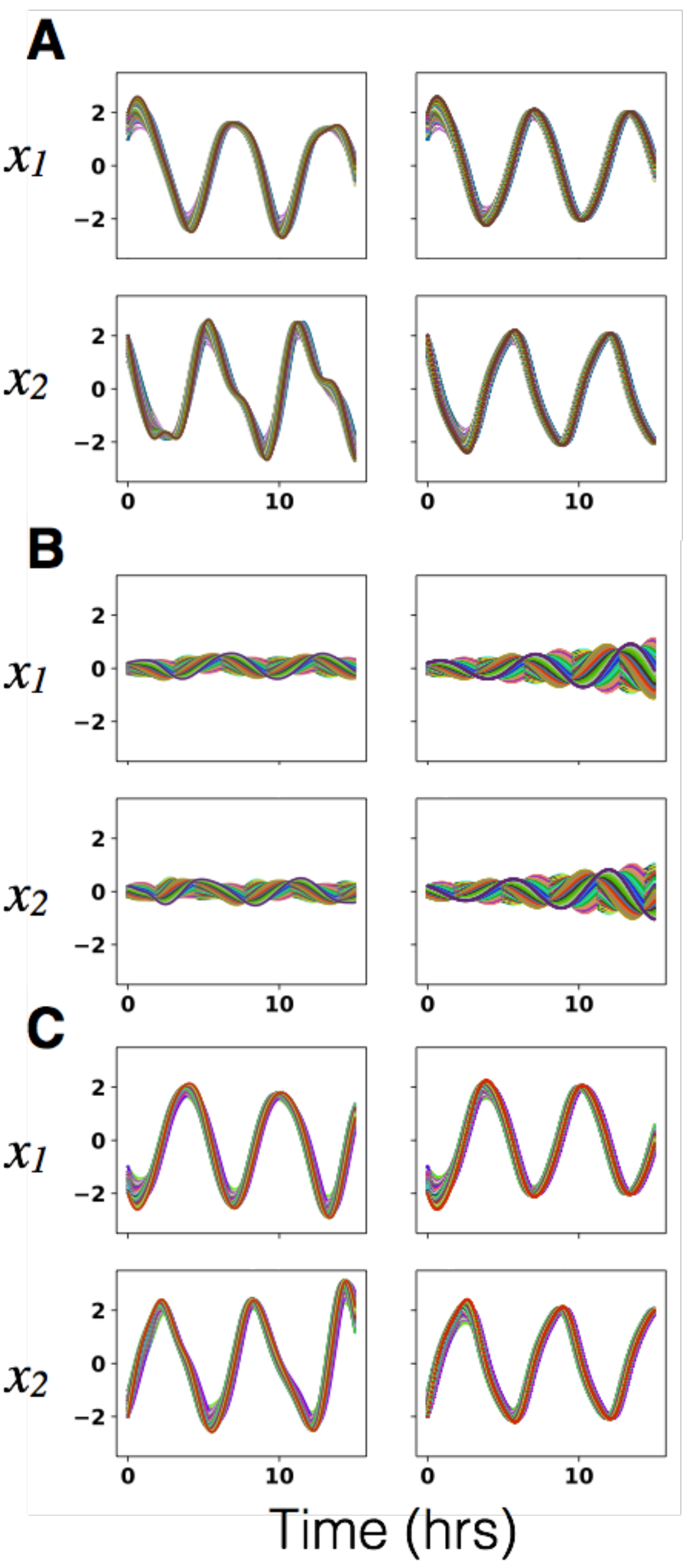}
\caption{The state t rajectories of the Van Der Pol oscillator on the left, Koopman approximations on the right.  (A) State evolution with positive initial conditions.   (B) State evolution with  initial conditions near $0.$ (C) The state trajectories of the Van Der Pol oscillator state evolution with negative initial conditions.}
\label{fig:vdp}
\end{figure}

\section{Convergence and Error Bounds}\label{sec:SILLerror}
Even if SILL functions satisfy finite approximate closure, it is necessary to evaluate the fidelity of their approximation.  We first show that fidelity of the approximation increases with the steepness parameter $\alpha$ and derive a global error bound. 
\subsection{Convergence in $\alpha$}
Without loss of generality  we let $v^l \gtrsim v^k$. Then the difference between each of the $n\times N_L$ terms in the summation is a scaling of the difference :
\begin{equation}\label{dualerrorterm}
\begin{aligned}
\alpha\Lambda_{v^l} & (x)\Lambda_{v^k}(x) - \alpha\Lambda_{v^l}(x) \\ 
&=\alpha\Lambda_{v^l} (x)\left(\Lambda_{v^k}(x) -1\right)\\ 
&=\alpha\frac{1-\Lambda_{v^k}(x)^{-1}}{(\Lambda_{v^l}(x)\Lambda_{v^k}(x))^{-1}} \\
&=\alpha \frac{1 -  (1+e^{-\alpha(x_1 - \mu^l_1)})   ... (1+e^{-\alpha(x_1 - \mu^l_n)}) }{\prod_{i=1}^{n}(1+e^{-\alpha(x_i - \mu^l_i)}) (1+e^{-\alpha(x_i - \mu^k_i)})} \\
&= \frac{\alpha}{\prod_{i=1}^{n}s_{il}(x)s_{ik}(x)} - \frac{\alpha}{\prod_{i=1}^n s_{il}(x)}
\end{aligned}
\end{equation}
where $s_{ij}(x) = (1+e^{-\alpha(x_i-\mu_i^j)})$.
Based on this error term we have the following theorem:
\begin{theorem}
Given that $x_i\neq\mu_i \forall i\in \{1, 2, ..., n\}$. As $\alpha \rightarrow \infty$, the error between $d\Lambda/dt $ and its approximation 
\[
\sum_{i=1}^{n}\sum_{k=1}^{N_L} \alpha(1 - \lambda_{\mu^l_i}(x_i)) w_{ik}  \Lambda_{v_{max}(l,k)}(x) 
\]
converges to $0$. 
\end{theorem}
\begin{proof}
The error term between $d\Lambda/dt$ and \[ \sum_{i=1}^{n}\sum_{k=1}^{N_L} \alpha(1 - \lambda_{\mu^l_i}(x_i)) w_{ik}  \Lambda_{v_{max}(l,k)}(x) \] is comprised of terms described by equation (\ref{dualerrorterm}), namely 
\begin{equation}
 \frac{\alpha}{\prod_{i=1}^{n}s_{il}(x)s_{ik}(x)} - \frac{\alpha}{\prod_{i=1}^n s_{il}(x)},
\end{equation}
where $s_{ij}(x) = (1+e^{-\alpha(x_i-\mu_i^j)})$.  We observe three cases.  In each of these cases we hold $x$ constant and allow $\alpha$ to vary.

\textbf{Case 1},  $x_i-\mu_i = 0$: So, $\frac{1}{1+ e^{-\alpha(x_i-\mu_i)}} = \frac{1}{2}$ for all $\alpha \in \R^+$. 

\textbf{Case 2}, $x_i-\mu_i > 0$: As $\alpha \rightarrow \infty$ we have that $e^{-\alpha(x_i-\mu_i)}\rightarrow 0$ and so $\frac{1}{1+ e^{-\alpha(x_i-\mu_i)}} \rightarrow 1$. 

\textbf{Case 3}, $x_i-\mu_i < 0$: As $\alpha \rightarrow \infty$ we have that $e^{-\alpha(x_i-\mu_i)}\rightarrow \infty$ and so $\frac{1}{1+ e^{-\alpha(x_i-\mu_i)}} \rightarrow 0$.  We further note that as $\alpha \rightarrow \infty$, $\frac{\alpha}{1+ e^{-\alpha(x_i-\mu_i)}} \rightarrow 0$.

Defining $S_p=\{1, 2, ..., p\}$ for  $p\in\N$, the cases above imply that if there exists $i\in S_n$ for every $ k\in S_{N_L}$,  so that $x_i-\mu^k_i < 0$, then (\ref{mv_logistic_prime}) goes to 0 as $\alpha \rightarrow \infty$.  

Furthermore, if $x_i-\mu^k_i > 0$ then $\forall i\in S_n$ we have that $x_i-\mu^l_i > 0$, this also implies that  (\ref{dualerrorterm}) goes to $\alpha-\alpha=0$ as $\alpha \rightarrow \infty$.  

Since, by assumption $x_i - \mu^k_l \neq 0$ we have that our error of each term goes to zero.  Thus the sum of each of these terms comes to zero as well, as they are each only multiplied by a constant with respect to $\alpha$.
\end{proof}
Almost everywhere (for $x_i \neq \mu_i$) the error converges to $0$. At $x_i = \mu_i$, there is a small error incurred due to the approximation of a product of two totally-ordered conjunctive logistic functions with the ``greatest'' element of the pair.  This error never goes to zero without introducing additional SILL functions to aid in the approximation.  

\subsection{Global Error Bounds}
Denote the error term in (\ref{dualerrorterm}) as  $E_{kl}(x)$. Given a fixed $\mu^l$, $\mu^k$ and $\alpha$, the error is bounded above by $M_{kl}\in\R$.  We calculate $M_{kl}$ by taking the derivative of (\ref{dualerrorterm}), a gradient, and setting each of its $n$ terms to zero:
\begin{equation}
\begin{aligned}
\nabla E_{kl}(x)_j = & \frac{\eta(x)}{s_{jl}(x)s_{jk}(x)\prod_{i=1}^n s_{il}(x)s_{ik}(x)} \\ 
&- \frac{\alpha^2 e^{-\alpha (x_j-\mu_j^l)}}{s_{jl}(x)\prod_{i=1}^n s_{il}(x)} = 0\\
\end{aligned}
\end{equation}
where 
\begin{equation}
\eta(x) \equiv \alpha^2(e^{-\alpha (x_j-\mu_j^l)}+e^{-\alpha (x_j-\mu_j^k)}+2e^{-\alpha (2x_j-\mu_j^l-\mu_j^k)}).
\end{equation}
We find a common denominator and multiply both sides by it then divide out $\alpha^2 e^{-\alpha (x_j-\mu_j^l)}$ to obtain: 

\begin{equation}
\begin{aligned}
0 =& 1+e^{-\alpha (\mu_j^l-\mu_j^k)}+2e^{-\alpha (x_j-\mu_j^k)} \\
&- (1+e^{-\alpha (x_j-\mu_j^k)})\prod_{i=1}^n (1+e^{-\alpha (x_i-\mu_i^k)})
\end{aligned}
\end{equation}
and we set $y_\ell = e^{-\alpha  x_\ell}$ resulting in a  multivariate polynomial:  
\begin{equation}
0 = p_j({\bf y }), \forall j\in\{1,2,\hdots,n\}
\end{equation}

We then have $n$ equations with $n$ unknown variables, we define our solution be the points: $r_1, r_2, ..., r_n$, then we consider the set of points $\mathcal{S}=\{s\in \R | s = \frac{\ln(r_i)}{-\alpha}\forall i\in \{1,2,\hdots,n\}\}$. We set $M_{kl} = \max_{s\in\mathcal{S}}\{E_{kl}(s)\}$.

Then the sum of the $ N_L \times n$ error terms,  $E_{kl}(x)$, we call $E_{\Lambda_l}(x)$. The sum of the $ N_L \times n$ maximal error terms,  $M_{kl}$, we call $M_{\Lambda_l}$. Since, by assumption, our error in approximating $f$ is zero, and the derivative of 1 is zero everywhere, we have that the total error in our Koopman approximation of the derivative of the state vector $x$ at time $t$ will be:\
\begin{equation}
\sum_{l = 1}^{N_L} E_{\Lambda_l}(x(t))
\end{equation}
Thus our error in estimating the state at time $t$ will be:
\begin{equation}
\begin{aligned}
\int_0^t\sum_{l = 1}^{N_L} & E_{\Lambda_l}(x(\tau)) d\tau \\
&\leq \int_0^t\sum_{l = 1}^{N_L}M_{\Lambda_l}d\tau \\
&= t \sum_{\ell = 1}^{N_L}M_{\Lambda_l}
\end{aligned}
\end{equation}

This holds true under the assumption that the approximation of the function, $f$ by logistic functions is a perfect approximation. In the case where there is error in the approximation of $f$ we have the following:

\begin{equation}
\dot{x_k} = f_k(x) = \delta_k(x) + \sum_{\ell=1}^{m}w_{k\ell}\Lambda_{v_l}(x)
\end{equation}
where $\delta_k(x)$ is the error when approximating $f$ at $x$.  The value of $\delta_k$ is tied to our mesh resolution parameter, $\epsilon$.  Thus, we have
\begin{equation}\label{mv_log_prime_w_error}
\begin{aligned}
\dot{\Lambda}_{v_l}(x)& = \sum_{i=1}^{n} \alpha(1 - \lambda_{\mu^l_i}(x_i)) \Lambda_{v_l}(x) f_i(x)\\ 
= \sum_{i=1}^{n} & \alpha(1 - \lambda_{\mu^l_i}(x_i)) \Lambda_{v_l}(x) \sum_{k=1}^{N_L}w_{ik} \Lambda_{v_k}(x)\\ 
 &+ \sum_{i=1}^{n} \delta_i(x)\alpha(1 - \lambda_{\mu^l_i}(x_i)) \Lambda_{v_l}(x) \\
= \sum_{i=1}^{n} & \sum_{k=1}^{N_L} \alpha(1 - \lambda_{\mu^l_i}(x_i)) w_{ik}  \Lambda_{v_l}(x) \Lambda_{v_k}(x) \\ 
 &+ \sum_{i=1}^{n} \delta_i(x)\alpha(1 - \lambda_{\mu^l_i}(x_i)) \Lambda_{v_l}(x) \\
\end{aligned}
\end{equation}
So, we add another $n$ error terms to our derivative approximation.  Each will be approximated by the multivariate logistic function $\Lambda_{v_l}$.  Thus the error of approximation of these terms for any given $k\in\{1,2,\hdots,n\}$ will be bounded by:
\begin{equation}
|\delta_k(x)(1-\lambda(x_k - \mu_k^l)) - 1| \leq |\delta_k(x) - 1|.
\end{equation}


Ultimately, the error when approximating the behavior of any element in the SILL basis is bounded.  We thus have that our choice of lifting has finite approximate closure which means that it may be used to extract stability properties.  We demonstrate this in two examples below.

\section{Numerical Examples}\label{sec:examples} 
\subsection{The Van der Pol Oscillator}

We consider the Van der Pol system.  This system features a stable limit cycle in the saddle region of its phase space.  The system thus presents a challenge, since it contains oscillatory and unstable dynamics all within the same phase space.  Arbabi and Mezic showed it was possible to use  Koopman representations to learn the asymptotic phase of the operator \cite{arbabi2016ergodic}. The equations for the system are below:
\begin{equation}
\begin{aligned}
\dot{x}_1 &= x_2 \\
\dot{x}_2 &= -x_1  + \alpha_1 (1 - x_1^2) x_2 
\end{aligned}
\end{equation}
where $\alpha_1$ is taken to be  $-0.2$ in all simulations. 

Our results show that we were able to learn the oscillatory dynamics in two regions of phase space (see Figure \ref{fig:vdp}A and \ref{fig:vdp}C).  However, the SILL functions were not able to predict the unstable dynamics of the Van Der Pol oscillator.  This was because the SILL functions had to be defined on a finite lattice.  The boundaries of the lattice incur the most error, since the vector field is not evaluated beyond the boundary region.  Specifically, for the Van der Pol oscillator, these boundaries coincided with unstable dynamics of the system.  Moreover, there is a numerical conditioning challenge with identifying a model with unstable modes.  

\subsection{The Bistable Toggle Switch}
We now consider a bistable toggle switch system as proposed in \cite{gardner2000construction}.  This system models the interaction between two proteins $LacI$ and $TetR$ who repress each other, resulting in one of two equilibrium points ultimately being reached depending on the initial concentrations of each of the two proteins. For simplicity, we refer to these proteins as protein 1 and protein 2 respectively. Given constants: $n_1, n_2, \alpha_1, \alpha_2, \delta \in \R$, the simplest model is a two state repression model \cite{gardner2000construction} of the form 
\begin{equation}
\begin{aligned}
\dot{x}_1 = \frac{\alpha_1}{1 + x_2^{n_1}} - \delta x_1 \\ 
\dot{x}_2 = \frac{\alpha_2}{1 + x_1^{n_2}} - \delta x_2 
\end{aligned}
\end{equation}
where $x_1, x_2$ are the concentrations of the respective proteins 1 and 2.  We note that given the proper parameters, and under a wide range of initial conditions, our SILL functions and their associated approximate Koopman generator correctly indicate the tendency of nearly every set of initial protein concentrations.  Specifically, the error in approximation was less than $1\%$ for both $f_1(x)$ and $f_2(x)$.

\begin{figure}[t]
\centering
\includegraphics[width=\columnwidth]{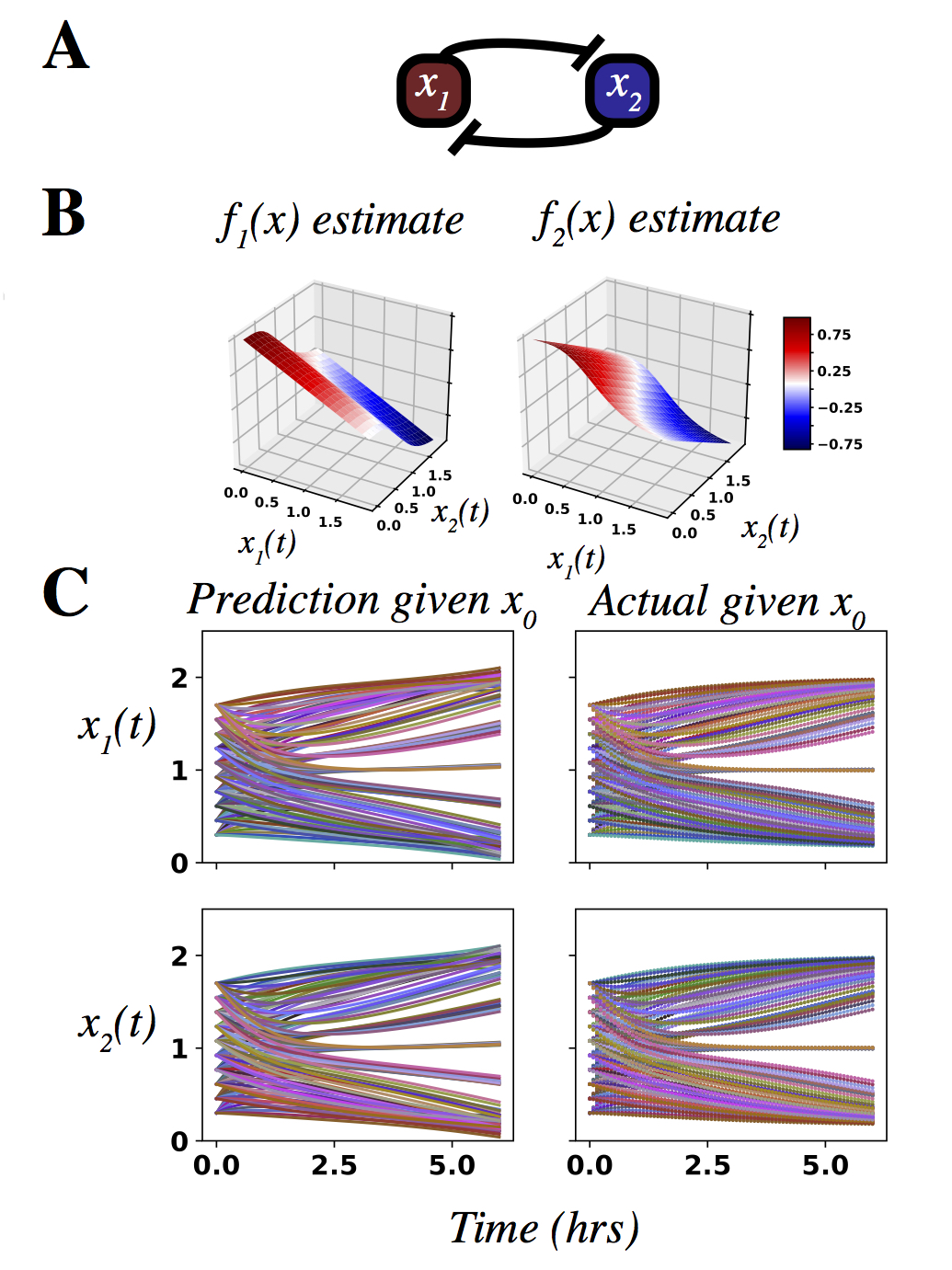}
\caption{(A) A diagram showing the mutual repressing architecture of the toggle switch from \cite{gardner2000construction}.  (B) The estimates for the vector field $f_1(x)$ and $f_2(x)$ generated by SILL basis of order 36.  The error in approximation was less than 1\% for both $f_1(x)$ and $f_2(x)$.  (C) Forward prediction of $x_1(t)$ and $x_2(t)$ given a distribution of initial conditions $x_0$. }
\label{fig:toggle_bistable}
\end{figure}

\section{Conclusions}

We set out to find finite dimensional approximations to Koopman generators for nonlinear systems.  We introduced a class of state-inclusive observable functions  comprised of products of logistic functions that confer an approximate finite closure property.   We derived error bounds for their approximation, in terms of a steepness $\alpha$ and a mesh resolution parameter $\epsilon.$   In particular, we show that introduction of SILL observable functions does not introduce unbounded error in the Koopman generator  approximation, since a mesh dictionary of SILL functions satisfies a total order property.  Further, the error bound can be reduced by modifying the learning parameters $\alpha$ and $\epsilon.$ 

In future work, we will study the use of structured regularization or structured sparse compressive sensing may result in a more efficient and concise set of Koopman dictionary functions. 

There are many scenarios where snapshots of the underlying system from different observers may each yield a scalable Koopman generator. The process of synthesizing or integrating these Koopman operators to obtain a global Koopman operator (coinciding with global measurements), is a subject of ongoing research.

\bibliographystyle{unsrt}
\bibliography{bibliography.bib}
 
\end{document}